\documentclass{arxiv} 

\title[Optimization dependent generalization bound for ReLU networks]{Optimization dependent generalization bound for ReLU networks based on sensitivity in the tangent bundle}

\optauthor{%
 \Name{Dániel Rácz} \Email{racz.daniel@sztaki.hu}\\
 \addr Institute for Computer Science and Control (SZTAKI) \& Eötvös Loránd University, Budapest, Hungary 
 \AND
 \Name{Mihály Petreczky} \Email{mihaly.petrecszky@univ-lille.fr}\\
 \addr CNRS CRIStAL, Université de Lille, Lille, France%
 \AND
 \Name{András Csertán} \Email{csertan.andras@sztaki.hu}\\
 \addr Institute for Computer Science and Control (SZTAKI) \& Eötvös Loránd University, Budapest, Hungary%
 \AND
 \Name{Bálint Daróczy} \Email{daroczyb@ilab.sztaki.hu}\\
 \addr Institute for Computer Science and Control (SZTAKI), Budapest, Hungary \& Széchenyi István University, Győr, Hungary
}

\newtheorem{assumption}[theorem]{Assumption}

\newcommand{\norm}[1]{\left\lVert#1\right\rVert}
\newcommand{\x}{\mathbf{x}}
\newcommand{\z}{z}
\newcommand{\vv}{\mathbf{v}}

\begin{document}

\maketitle

\begin{abstract}%
Recent advances in deep learning have given us some very
promising results on the generalization ability of deep neural networks, however
literature still lacks a comprehensive theory explaining why heavily
over-parametrized models are able to generalize well while fitting the training
data.  In this paper we propose a PAC type bound on the generalization error of
feedforward ReLU networks via estimating the Rademacher complexity of the set of
networks available from an initial parameter vector via gradient descent. The
key idea is to bound the sensitivity of the network's gradient to perturbation
of the input data along the optimization trajectory. The obtained bound does
not explicitly depend on the depth of the network. Our results are
experimentally verified on the MNIST and CIFAR-10 datasets. 
\end{abstract}


\section{Introduction and related work}
Deep learning has started soaring in popularity during the last decade and by
today it has reached unprecedented heights in terms of practical usage as well
as theoretical research. As higher computational capacity is also becoming
easier to access, the complexity and size of deep neural networks being used by
the community is also dramatically increasing. The number of parameters
contained in such networks are usually much higher than the number of training
data points or the dimension of the training data. As a result, these models
tend to easily interpolate the training data and according to the "classical"
theory of machine learning (see e.g. \cite{vapnik1999nature}) they should lead
to overfitting. However, it has been shown \cite{zhang2021understanding}
empirically that this is not the case, i.e. highly over-parametrized models
trained by gradient descent are capable of generalizing well while fitting the
training data. 
This phenomenon drove the theoretical research of deep learning towards
examining over-parametrized models. A big push has been given by the discovery of
the Neural Tangent Kernel (NTK, \cite{jacot2018neural}), leading to several
convergence theorems in the infinitely wide regime
\cite{du2018gradient,allen2019convergence,liu2022loss} and shedding light on the
connection of over-parametrized networks and kernel machines
(see \cite{belkin2021fit} for a comprehensive overview). While some interesting 
related
results
\cite{bowman2022spectral,radhakrishnan2022wide,li2023statistical,oymak2019generalization,vakili2021uniform} 
exist, it is not
entirely clear how the properties of the NTK might explain generalization. 

Deep networks are almost always trained using some form of the gradient descent
algorithm which seems to contain a so-called implicit bias
 in case of over-parametrized neural networks, i.e.
finding well generalizing optima despite interpolating the training data.
 The conditions and properties of implicit bias
is still under heavy research \cite{vardi2023implicit}. In
\cite{frei2023double} the authors show that under special conditions
on the data, gradient descent finds a solution which generalizes well, but the resulting model is
sensitive to adversarial perturbations. Analyzing the trajectories of gradient
descent has also been present in the literature
\cite{zou2020gradient,NIPS2015_5797}.

One of the standard methods to obtain generalization bounds for models in
statistical learning is the Probably Approximately Correct (PAC) framework \cite{mcallester1999pac}, which was
applied to deep networks in \cite{langford2001not} and later in
\cite{dziugaite2017computing,neyshabur2017exploring}. These papers establish
PAC-Bayesian bounds on the generalization error based on the estimation of the KL
divergence of the predictor w.r.t. some prior distribution of the parameters.
This was further developed in \cite{neyshabur2017pac} resulting in a bound
depending on the norm of the weights and (implicitly) on the depth of the
network. PAC bounds on the generalization error are closely connected to
 bounding the Rademacher complexity.
 In \cite{dziugaite2017computing} and \cite{bartlett2017spectrally}
 the authors also exploited some bounds on the Rademacher complexity of the underlying family of functions represented
by ReLU networks. Several other bounds on the Rademacher complexity was derived in 
\cite{golowich2018size} depending on yet again the norm of the weight matrices
and the width and depth of the network. In \cite{shilton2023gradient} a bound on
the Rademacher complexity is achieved by defining the
learning problem in a Reproducing Kernel Banach Space (RKBS) under some
conditions on the learning algorithm. Generalization bounds for deep
convolutional networks and Graph Neural Networks have been established in 
\cite{liao2020pac} and \cite{long2019generalization}, respectively.

\subsection{Our contribution}

Informally, our main result is an upper bound on the generalization error of
feedforward ReLU networks trained with gradient descent under certain conditions,
depending on the optimization path. Before starting our
journey to precisely state our theorem, we give a brief overview of the most
important ingredients of our concept in order to make the rest of the article
more traceable.
\begin{itemize}
  \item We will make use of a classical PAC inequality for the generalization
  error which depends on the Rademacher complexity of the loss values
   on some (test) sample.
  \item The key step is to upper bound this Rademacher complexity
   of the family of functions represented by ReLU
  networks by examining
  the network's behavior along the optimization trajectory
   from the perspective of sensitivity to perturbation of
   the input data.
  \item The basic idea behind this sensitivity measure is the following. We look
  at the gradient of the network function w.r.t. the parameters as a feature map
  defined on the input data. What happens to this feature representation if the
  input data is perturbed by some Gaussian noise?
  \item We reinforce our theoretical bound by performing experiments on the
  MNIST
   \cite{lecun1998gradient} and
  \\CIFAR-10 \cite{krizhevsky2009learning} datasets.
  \footnote{The implementation is available at \url{https://github.com/danielracz/tansens_public}}
\end{itemize}
The idea of measuring the change in the network's gradients caused by Gaussian
perturbation of the input data has been introduced in
\cite{daroczy2022gaussian}, where it was empirically shown that in case of the
task of classification, the resulting measure
correlates with the generalization gap of the network and can be used to
estimate the test loss without making use of the labels of the test data.
However, previously there was no theoretical connection known to us 
explaining this phenomenon. 

Exploiting the representation of the data via the gradient of the network as a
feature map is one of the underlying ideas of the Neural Tangent Kernel and has
been seen before (e.g. \cite{cao2019generalization,liu2023emergence}).
In \cite{racz2021gradient} such representation is used to induce a similarity
function on the data. While the gradient is usually constant over training
\cite{liu2023emergence} under heavy over-parametrization, we suspect it has a vital connection to the
generalization ability of the network in both the finite and infinite case.

\section{Problem setup}
\subsection{Notations}
Let us consider the framework of Empirical Risk Minimization over the task of
binary classification, i.e. we are given a finite set of training data
 $D = \{(\x_i,y_i) ; i=\{1,\dots,n\}\}$ drawn from a probability distribution
 $\mathcal{X} \times \mathcal{Y}$ on $\mathbb{R}^{n_{in}} \times \{-1,1\}$.
 Let $\mathcal{L}_{emp}^{D}(f) = \frac{1}{n}\sum_{i=1}^{n} l(f(\x_i),y_i)$ denote the empirical
loss we want to minimize defined over a class of functions $\mathcal{F}$. We
denote the true error by $\mathcal{L}(f) =
\mathbf{E}_{(\x,y)\text{$\sim$}\mathcal{X} \times \mathcal{Y}}[l(f(\x), y]$.
The generalization error or gap of a model $f \in \mathcal{F}$ is defined as 
$|\mathcal{L}_{emp}^{D}(f) - \mathcal{L}(f)|$. In practice, we can approximate the
generalization gap by the empirical generalization gap, i.e. the loss difference
on the training data and some test data (see \cite{devroye2013probabilistic}).

Let the function class $\mathcal{F}$ we would like to optimize over be a family
of ReLU networks characterized by a parameter vector $\theta \in \mathbb{R}^{P}$.
We will treat such networks as a function of both the input data
and 
the parameter vectors denoted by
 $f: \mathbb{R}^{P} \times \mathbb{R}^{n_{in}} \rightarrow \mathbb{R}^{n_{out}}$,
where $P$ is the number of parameters and $n_{in}$ is the input dimension. As we
are dealing with the binary classification task we have $n_{out} = 1$.
Such models are usually trained by using the gradient descent algorithm from the
initial point $\theta_0$ defined recursively at time $T$ as
$\theta_T = \theta_{T-1} - \eta(T) \nabla_{\theta}
\mathcal{L}_{emp}^{D}(f(\theta_{T-1},\cdot))$, where $\eta(T) \in
\mathbb{R}^{+}$ is the learning rate at time $T$. We follow the convention of 
$\text{ReLU}'(0) = 0$. For a fixed choice of the
learning rate function $\eta: \mathbb{N} \rightarrow \mathbb{R}^{+}$ we say 
$\theta = GD(\theta_0, \eta, T)$, if $\theta$ is the output of the gradient
descent after $T$ steps initialized in $\theta_0$ and run with the choice
of $\eta$ as the learning rate. Let $Traj(\theta_0)$ denote the set of
parameter vectors $\theta$ for which there exists $\eta$ and $T$ such that
 $\theta = GD(\theta_0, \eta, T)$.
\subsection{Tangent Sensitivity}
The central definition of our paper is called Tangent Sensitivity. Initially it
was defined in \cite{daroczy2022gaussian} motivated by the following.
Consider a small enough Gaussian perturbation around
$\x \sim \mathcal{X}$ with $\phi(\x) = \x + \delta(\x)$ where
$\delta(\x) \sim
\mathcal{N}(0,\sigma\mathbf{I})$ is a random variable.
For the expected change in the gradient mapping defined as 
$\x \rightarrow \nabla_\theta f(\theta, \x)$ on the input space we have
\begin{align*}
 \mathbf{E}_{\delta(\x)} [\left \| \nabla_\theta f(\theta, \x)  - \nabla_\theta
  f(\theta, \phi(\x)) \right \|_2^2 ] \sim \mathbf{E}_{\delta(\x)} \left
   [ \left \|
  \frac{\partial{\nabla_\theta f(\theta, \x)}}{\partial{x}} \delta(\x)
   \right \|_2^2 \right ]
   \leq \sigma \left \|
   \frac{\partial{\nabla_\theta f(\theta, \x)}}{\partial{x}}\right \|_2^2.
\end{align*}
The first approximation is based on the Taylor expansion of the gradient mapping
and scales with the variance $\sigma$ of the Gaussian noise. Hence the next
definition.
\begin{definition}
The Tangent Sample Sensitivity of a feedforward network $f$ 
  with output in $\mathbb{R}$ at input $x \in \mathbb{R}^{n_{in}}$
  is a $P \times n_{in}$ dimensional matrix,
  $S(\theta, \x) := \frac{\partial\nabla_\theta f}{\partial x}(\theta, \x)
  = \frac{\partial^2{f}}{ \partial{x}\partial{\theta}} 
  (\theta, \x)$.
  The Tangent Sensitivity is the expectation of
  tangent sample sensitivity, i.e. $S(\theta):=\mathbf{E}_{\x \sim \mathcal{X}}[S(\theta, \x)]$.
\end{definition}
Among other interesting properties it is empirically shown in
 \cite{daroczy2022gaussian} that the Frobenius norm of the Tangent Sensitivity
 matrix has a close relationship to the generalization error of the network.
 Theoretical explanation of this phenomenon has been unknown to us in the
 literature, thus to address this experience,
  we will establish a PAC bound on the generalization
 gap in which the norm of the Tangent Sensitivity appears. \\
 During the rest of the paper we will abuse the naming convention and shorten
 Tangent Sample Sensitivity to Tangent Sensitivity in some cases.
\section{PAC bound}
Our goal now is to state our main theorem. First we need to introduce a series
of assumptions.
\begin{assumption}
  \label{ass:1}
  The loss function
   $ l$ has the form $ l(f(\x),y) = \ell(f(\x) - y)$, where
   $\ell$ is $K_{\mathcal{L}}$-Lipschitz.
\end{assumption}
\noindent This is a mild assumption as most of the standard loss functions are Lipschitz
on a bounded domain.
\begin{assumption}
  \label{ass:2}
  For a fixed $\theta_0 \in \mathbb{R}^{P}$ and $\varepsilon > 0$ let
  $U_{\theta_0, \varepsilon} = Traj(\theta_0) \cap B_{\varepsilon}(\theta_0)$,
  where $B_{\varepsilon}(\theta_0)$ is the $\varepsilon$-ball around $\theta_0$.
  We assume that for any $\theta \in U_{\theta_0, \varepsilon}$ the Frobenius
  norm of the Tangent Sensitivity is bounded on the training set,
  i.e. $\sup\limits_{\x \sim D}
  \norm{S(\theta, \x)}_{F} \leq C_{TS}$ for some $C_{TS} > 0$.
\end{assumption}
\noindent Empirical evidence suggests that this is a reasonable assumption around
an initialization point $\theta_0$, which in practice usually contains an
optimum.
In light of Assumption \ref{ass:2}  we define the family of models 
$\mathcal{F}_{\theta_0, C, \varepsilon} := \left \{ f(\theta, \cdot) \text{ }
\middle |
\text{ } \theta \in
U_{\theta_0, \varepsilon} \text{ , } \sup\limits_{\x \sim D}\norm{S(\theta,
 \x)}_{F} \leq C \right \}$.
 \begin{assumption}
  \label{ass:3}
  We assume the following upper bounds hold:
       $ \sup\limits_{x \sim \mathcal{X}} |f(\theta_0, \cdot)| \leq K_{\theta_0}$, \\
       $\sup\limits_{x \sim \mathcal{X}} \norm{\nabla_{\theta}f(\theta_0, \x)}
        \leq K_{\nabla_0}$
       and 
      $\sup\limits_{x \sim \mathcal{X}} \norm{\x} \leq K_x$.
  \end{assumption}
\noindent Note, that the these assumptions are widely applied in the literature.
The first two refer to the boundedness of the function and its gradient
around the initialization, while the third one assumes the input is bounded.
%
\begin{theorem}
  \label{theorem:main}
  Consider the problem setup from Section 2 and let Assumption
   \ref{ass:1} - \ref{ass:3} hold for a fixed initialization $\theta_0$ and
   $\varepsilon > 0$. Furthermore let us suppose for all $T \in \mathbb{N}$
    that along all gradient descent
   trajectories of length $T$ starting from $\theta_0$ the quantity
   $\sum\limits_{t=1}^{T - 1} \eta(t)
   \frac{1}{n}\sum\limits_{i=1}^{n}\frac{\partial l}{\partial f}(\theta_{t-1}, \x_i)$
    is upper bounded by a positive constant $C_{GD}$. Then
    for any $\delta \in ]0,1[$ with probability at least $1 - \delta$ over the
    random sample $S$ we have
 \begin{align*}
  \forall f(\theta, \cdot ) \in \mathcal{F}_{\theta_{0}, C_{TS}, \varepsilon}:
  \mathcal{L}(f) - \mathcal{L}_{emp}^{S}(f)
   \leq K_{\mathcal{L}} K_{\theta_0} + \frac{C_1}{\sqrt{N}} + K_{\mathcal{L}} H(\theta) 
   + B \sqrt{\frac{2 log (\frac{4}{\delta})}{N}},
\end{align*}
where $C_1 = 2K_{\mathcal{L}}K_{x}K_{\nabla_0}C_{TS}C_{GD}$ and
 $H(\theta)$ is an error term and $N$ denotes
the size of $S$. The constant
$B$ is an upper bound on the loss $l(\cdot, \cdot)$.
Additionally, along with a properly scaled ReLU activation if the width of the
 network tends to infinity, $H(\theta)$ tends to zero.
\end{theorem}
\begin{figure}
\centering
\includegraphics[width=\textwidth]{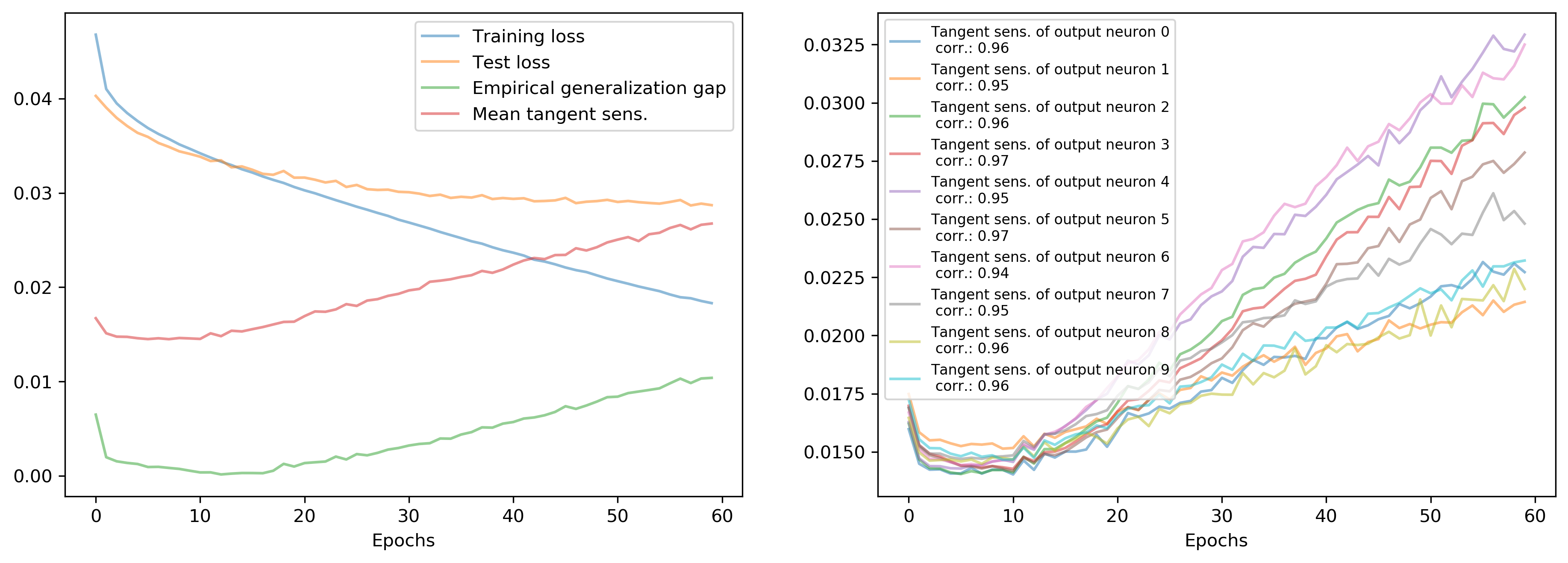}
\caption[]{Pearson correlation\protect\footnotemark{} between the empirical
 generalization gap and the
average norm of the Tangent Sensitivity on the test dataset
for a $3 \times 3000$-wide fully connected ReLU net trained on CIFAR-10.
The y-axes are based on the actual loss values.
The norm values of tangent sensitivity were linearly scaled for presentational
purposes. 
For more details and more experiments see Appendix \ref{appendix:C}.}
\label{fig:wide}
\end{figure}
There are two critical terms in the bound of Theorem  \ref{theorem:main}, namely
$C_{TS}C_{GD}$ and $H(\theta)$. The former one highlights the connection of the
generalization ability of the network and the Tangent Sensitivity along the
optimization trajectory and it originates from the estimation of the Rademacher
complexity (Definition \ref{defiradem} in Appendix \ref{appendix:A}) of
the network on the sample $S$. While currently  we do not have a satisfying
theoretical guarantee on the value of $C_{TS}$, we have
strong empirical evidence that the norm of the Tangent Sensitivity 
is indeed correlated to the empirical generalization gap, see
Fig.~\ref{fig:wide}.
\footnotetext{\url{https://docs.scipy.org/doc/scipy/reference/generated/scipy.stats.pearsonr.html}}
 For more details on our
experiments see Appendix \ref{appendix:C}.

The second term comes from the well known Taylor-approximation of the network and
it is proportional to the norm of the Hessian of the network function. The main
intuition behind the proof lies
 in the following possible approximation
of a ReLU network. If $\theta = \theta_T$, then
\begin{align*}
      f(\theta_T, \x) 
        &= f(\theta_0, \x) - \nabla_{\theta}f(\theta_0, \x)^T
      \left (\sum\limits_{t = 1}^{T - 1}\eta(t)
        \frac{1}{n}\sum\limits_{i=1}^{n}\frac{\partial l}{\partial f}
        S(\theta_{t - 1}, \x_i)\x_i \right )
        + h(\theta_T, \x),
\end{align*}
where $h$ is an error term (see Appendix \ref{appendix:B}) which determines
the term
 $H(\theta)$ in
Theorem \ref{theorem:main}.  Let 
 $w_{\theta} = \sum\limits_{t = 1}^{T - 1}\eta(t)
        \frac{1}{n}\sum\limits_{i=1}^{n}\frac{\partial l}{\partial f}
        S(\theta_{t - 1}, \x_i)\x_i $.
Because $w_{\theta}$ depends only on the training set and
the optimization path, but not the actual input $\x$, we can approximate
 a ReLU network by the scalar product $<w_{\theta},\nabla_\theta f(\theta_0, \x)>$ and bound the norm of
 $w_\theta$  thanks to the various assumptions.
Finally, we can apply the standard techiniques for bounding the Rademacher
complexity of linear classifiers.
For a
complete proof of Theorem \ref{theorem:main} see Appendix \ref{appendix:B}.

\section{Discussion and future work}

In this paper we have established a PAC bound on the generalization error of 
feedforward ReLU networks, crucially depending on the Tangent Sensitivity along
the optimization trajectory. Empirical evidence has previously shown the
correlation between the two quantities, we believe the obtained bound provides
a strong theoretical justification. While the established bound is not tight,
we believe that the sensitivity measure might have a connection to
the smoothness of the function in some appropriate function space
\cite{belkin2021fit},
  which is a promising direction for the generalization theory of deep networks. 
The straightforward next step seems to be the convergence analysis of the
 Tangent Sensitivity matrix and its norm around the initialization and optima.
  An interesting idea would be to incorporate it in the loss function, as presented
in \cite{rivasplata2019pac}, however calculating the Tangent Sensitivity norm is
 computational expensive, thus it requires to find an efficient approximation.
In the infinite width limit, Tangent Sensitivity can be viewed as a partial
derivative-like object of the NTK w.r.t. the input data. It would be interesting
to examine the connection to the generalization ability of the NTK Kernel
Machine.

\section{Acknowledgement}
This research was supported by the European Union project RRF-2.3.1-21-2022-00004 within the framework of the Artificial Intelligence National Laboratory and by the C.N.R.S. E.A.I.  project "Stabilité des algorithmes d'apprentissage pour les réseaux de neurones profonds et récurrents en utilisant la géométrie et la théorie du contrôle via la compréhension du rôle de la surparamétrisation". B.D. was supported by MTA Premium Postdoctoral Grant 2018. 


\newpage
\bibliography{tangent}
\newpage
\clearpage
\appendix

\section{Rademacher complexity}
\label{appendix:A}
\begin{definition}(e.g. see definition 26.1 in \cite{shalev2014understanding})
  \label{defiradem}
  The Rademacher complexity of a bounded set
   $\mathcal{A} \subset \mathbb{R}^{m}$ of vectors is defined as
      $R(\mathcal{A}) = \mathbb{E}_{\mathbf{\sigma}}\Bigg[\sup_{a \in \mathcal{A}}
      \frac{1}{m} \sum\limits_{i = 1}^{m}\sigma_i a_i \Bigg]$,
  where the random variables $\mathbf{\sigma}_i$ are i.i.d such that 
  $\mathbb{P}[\sigma_i = 1]  = \mathbb{P}[\sigma = -1] = 0.5$. The Rademacher complexity of a set of functions $\mathcal{F}$ over a set of
  samples $S = \{s_1, \dots, s_m\}$ is defined as
$      R_{S}(\mathcal{F}) = R(\left\{(f(s_1),\dots,f(s_m)) \middle
       | f \in \mathcal{F} \right\})$.
\end{definition}

\section{Proof of Theorem \ref{theorem:main}}
\label{appendix:B}
In order to prove Theorem \ref{theorem:main} we will apply the following
general PAC bound to our situation.
\begin{theorem}[e.g. see Theorem 26.5 in \cite{shalev2014understanding}]
  \label{theorem:helper}
 Let $\mathcal{F}$ be a compact set of hypotheses. For any $\delta \in ]0,1[$
 \begin{align*}
  \mathbb{P}_{\mathcal{S}}\Bigg(\forall f \in \mathcal{F}:
  \mathcal{L}(f) - \mathcal{L}_{emp}(f)
   \leq 2 R_{S}(L_{0}) + B \sqrt{\frac{2 log (\frac{4}{\delta})}{N}}\Bigg) \geq 1 - \delta,
\end{align*}
where $R_{S}(L_{0})$ is the Rademacher complexity of the set
\\$L_{0} := \{ (l(f(\x_1, y_1)),\dots,l(f(\x_N,y_N))\text{ } |\text{ } f \in
\mathcal{F}) \}$,
$B$ is an upper bound on $l(\text{ } . \text{ } , \text{ } . \text{ } )$ and
the probability is taken over the random samples $S$ of size $N$.
\end{theorem}
Proof of Theorem \ref{theorem:helper} can be found in \cite{shalev2014understanding}. \\
First we take the Taylor approximation of $f(\theta, \x)$ around a random
initialization point $\theta_0$.
\begin{align*}
    f(\theta, \x) = f(\theta_0, \x) + \nabla_{\theta}f(\theta_0, \x)^T
    (\theta - \theta_0) + h(\theta, \x)
\end{align*}
where the approximation error $h(\theta, \x)
 = O((\theta - \theta_0)^T H_{\theta}f(\theta, \x)^T (\theta - \theta_0))$.
We know for wide networks and a suitably scaled ReLU activation
function the norm of the Hessian tends to zero as the width tends to infinity
 \cite{liu2022loss} in $B_{\varepsilon}(\theta_0)$,
 hence for a wide enough network the error term $h$ is sufficiently small.
 The radius $\varepsilon$ is proportional to the smallest eigenvalue of the NTK
 Gram matrix over the training data.
 Now consider the vanilla gradient descent for an $L_{\mathcal{L}}$-Lipschitz loss
  function of the form
 $\mathcal{L}_{emp}^{D}(\theta) = \frac{1}{n}\sum\limits_{i=1}^{n} l(f(\theta, \x_i), y_i)$ for a fixed set of
 training data. Note, that
  $\nabla_{\theta}\mathcal{L}_{emp}^{D}(\theta) = \frac{1}{n}\sum\limits_{i=1}^n \frac{\partial l}{\partial f}
  \nabla_{\theta}f(\theta, \x_i)$. According to the GD update rule
  \begin{align*}
      \theta_{T} = \theta_{T - 1} - \eta(T) \nabla_{\theta}\mathcal{L}_{emp}^{D}(\theta_{T - 1})
  \end{align*}
  where $\eta(T)$ is the learning rate at time $T$. 
  
  \begin{lemma}
    \label{lemma:relu}
    For a biasless ReLU network
    $\nabla_{\theta}f(\theta, \x) = S(\theta, \x) \x$.
  \end{lemma}

  \begin{proof}
    Let us fix an input vector $\x \in \mathbb{R}^{n_{in}}$ and a parameter
    vector $\theta \in \mathbb{R}^P$. We say that a path $p$ in the network
    graph is active if all the nodes in the path are active, meaning that the
    preactivation of every node is positive. The notation $p: x_j
    \rightarrow$ means that $p$ is a path starting at the $j$-th input neuron.
    We do not specify the output neuron as we are in the binary classification
    setting, where $n_{out} = 1$. Let $S_{i}$ be the $i$-th row of $S(\theta,
    \x)$.
    Since we fixed the input and parameter
    vectors, we will omit the dependence of them in the notation of the
    derivatives (as we did in case of $S_i$).
    
     The $i$-th coordinate of the LHS is
     \begin{align*}
      \frac{\partial f}{\partial \theta_i}  = 
      \sum\limits_{\substack{x_{j}: \text{input} \\
      \text{node}}}\sum\limits_{\substack{p: x_{j} \rightarrow \\
      \text{active path,} \\ \theta_i \in p}}\x_{j}\prod\limits_{\substack{\theta_{p} \in p \\
      \theta_p \neq \theta_i}}\theta_{p} = 
      \sum\limits_{\substack{x_{j}: \text{input} \\
      \text{node}}}\Big(\sum\limits_{\substack{p: x_{j} \rightarrow \\
      \text{active path,} \\ \theta_i \in p}}\prod\limits_{\substack{\theta_{p} \in p \\
      \theta_p \neq \theta_i}}\theta_{p} \Big ) \x_j = S_i \x
     \end{align*}
  \end{proof}

  Let us fix an initial parameter $\theta_0$ and a positive radius $\varepsilon$
   and let $U_{\theta_0, \varepsilon} = Traj(\theta_0)
    \cap B_{\varepsilon}(\theta_0)$. For any $\theta
   \in U_{\theta_0, \varepsilon}$ there is a $T$ such that
    $\theta = GD(\theta_0, \eta, T)$ for some
   $\eta$. Under these circumstances we may emphasize the relationship between
   $\eta$, $\theta$ and $T$ by the notations $\theta = \theta_{T}$ and 
   $\eta = \eta_{\theta_T}$. Using the GD update rule for such $\theta_T$ and
   the Taylor approximation of $f$ we obtain
   \begin{equation}
    \label{eq:1}
    \begin{aligned}
        f(\theta_T, \x) &= f(\theta_0, \x) - \nabla_{\theta}f(\theta_0, \x)^T
        \left (\sum\limits_{t = 1}^{T - 1}\eta(t)
         \nabla_{\theta}\mathcal{L}_{emp}^{D}(\theta_{t - 1})\right )
         + h(\theta_T, \x) \\
         &= f(\theta_0, \x) - \nabla_{\theta}f(\theta_0, \x)^T
        \left (\sum\limits_{t = 1}^{T - 1}\eta(t)
         \frac{1}{n}\sum\limits_{i=1}^{n}\frac{\partial l}{\partial f}
          \nabla_{\theta}f(\theta_{t - 1}, \x_i) \right )
         + h(\theta_T, \x) \\
         &= f(\theta_0, \x) - \nabla_{\theta}f(\theta_0, \x)^T
        \left (\sum\limits_{t = 1}^{T - 1}\eta(t)
         \frac{1}{n}\sum\limits_{i=1}^{n}\frac{\partial l}{\partial f}
          S(\theta_{t - 1}, \x_i)\x_i \right )
         + h(\theta_T, \x).
    \end{aligned}
   \end{equation}

The last equality follows from Lemma \ref{lemma:relu}. 
Note, that the first term and $\nabla_{\theta}f(\theta_0, \x)$ depend only on $\theta_0$.
 In order to use 
Theorem \ref{theorem:helper} we need to upper bound the
Rademacher complexity of \\
 $\left \{ (l(f(\theta, z_1)), \dots, l(f(\theta, z_N)))^T \middle |
 f(\theta, .) \in \mathcal{F}_{\theta_0, C_{TS}, \varepsilon}\right \}$ 
 for which it is enough to upper bound \\
$R_{S}(\left \{ (f(\theta, z_1), \dots, f(\theta, z_N))^T \middle |
 f(\theta, .) \in \mathcal{F}_{\theta_0, C_{TS}, \varepsilon}\right \})$ by Lemma 26.9 in
\cite{shalev2014understanding} and Assumption \ref{ass:1}. Here we denoted the
elements of the random sample $S$ from Theroem \ref{theorem:main}
by $\z_1,\dots,\z_N$ and incorporated the labels into the loss function $l$.
 Using the notation
 \\ $R_{S}(\mathcal{F}_{\theta_0, C_{TS},
 \varepsilon})
 = R_{S}(\left \{ (f(\theta, z_1), \dots, f(\theta, z_N))^T \middle |
 f(\theta, .) \in \mathcal{F}_{\theta_0, C_{TS}, \varepsilon}\right \})$,
Theorem \ref{theorem:main} immediately follows from the following proposition.

\begin{proposition}
  \label{prop:last}
  Under Assumptions \ref{ass:1}-\ref{ass:3} and the additional assumptions
  stated in Theorem \ref{theorem:main} we have
  \begin{align*}
      R_{S}(\mathcal{F}_{\theta_0, C_{TS}, \varepsilon})
       &\leq K_{\mathcal{L}}R_{S}(\left \{ (f(\theta_0, z_1) \dots f(\theta_0, z_M))^T \right \})
       + K_{\mathcal{L}}\frac{K_x K_{\nabla_0}C_{TS} C_{GD}}{\sqrt{N}}
       +  \\
       &+ K_{\mathcal{L}}R_{S}(\left \{ (h(\theta, z_1), \dots, h(\theta, z_M))^T \middle |
       f(\theta, .) \in \mathcal{F}_{\theta_0, C_{TS}, \varepsilon} \right \}).
  \end{align*}
\end{proposition}

\noindent
Before starting the proof we need another lemma.

\begin{lemma}[Lemma 26.10 from \cite{shalev2014understanding}]
  \label{lemma:linear}
  Let $\vv_1,\dots,\vv_m$ be vectors in a Hilbert space. Define
  $S' = \left\{ (\langle \mathbf{w}, \vv_1 \rangle,\dots,\langle \mathbf{w}, \vv_m
  \rangle)\text{ } \middle| \text{ } \norm{\mathbf{w}}_{2} \leq 1\right\}$.
  Then,
  \begin{align*}
    R(S') \leq \frac{\max_i \norm{\vv_i}_2}{\sqrt{m}}
  \end{align*}
\end{lemma}

\noindent
Note that by omitting the condition $\norm{\mathbf{w}}_{2} \leq 1$ the lemma
remains true with the modification of the upper bound multiplied by $\norm{\mathbf{w}}_{2}$.

\begin{proof}{\textbf{of Proposition \ref{prop:last}}}
  
\noindent
We remind the reader that the training dataset is considered
constant and is denoted by \\ $D = \{\x_1,\dots,\x_n\}$.
The three terms in the RHS of Proposition \ref{prop:last} respectively
originates from the three terms in equation \ref{eq:1} together with the fact
that for a set of the form $T = \left\{ A(\theta) + B(\theta) \text{ }
\middle| \text{ } \theta \in \Theta \right\}$ the Rademacher complexity is upper
bounded by 
\\$R_{S}(T) \leq R_{S}(\left\{ A(\theta) \text{ }
\middle| \text{ } \theta \in \Theta \right\})
 + R_{S}(\left\{ B(\theta) \text{ }
\middle| \text{ } \theta \in \Theta \right\})$.

The first and third terms are straightforward. The second term
comes from the application of Lemma \ref{lemma:linear} along with the roles
 $\vv_i = \nabla_{\theta}f(\theta_0, \z_i)$ and \\
  $\mathbf{w}_{\theta} = \frac{1}{n}\sum\limits_{t = 1}^{T - 1}\eta(t)
    \sum\limits_{i=1}^{n}\frac{\partial l}{\partial f}(\theta_{t-1}, \x_i)
\nabla_{\theta}f(\theta_{t - 1}, \x_i)$.
Assuming that the Tangent Sensitivity is bounded for every input along the
optimization path by $C_{TS}$ (Assumption \ref{ass:2})
 and the input is also bounded by $K_x$
(Assumption \ref{ass:3}), we have
\begin{align*}
\norm{\mathbf{w}_{\theta}} \leq C_{TS}K_{x} \sum\limits_{t=1}^{T-1}\eta(t)
\frac{1}{n}\sum\limits_{i=1}^{n}\frac{\partial l}{\partial f}(\theta_{t-1}, \x_i)
\leq C_{TS}K_{x}C_{GD},
\end{align*}
where we used the condition in Theorem \ref{theorem:main}.
\end{proof}
The term $H(\theta)$ in the Theorem is any upper bound on the Rademacher
complexity of \\
$R_{S}(\left \{ (h(\theta, z_1), \dots, h(\theta, z_M))^T \middle |
       f(\theta, .) \in \mathcal{F}_{\theta_0, C_{TS}, \varepsilon} \right \})$
. Additionally, we have \\
$R_{S}(\left \{ (f(\theta_0, z_1) \dots f(\theta_0, z_M))^T \right \}) \leq
K_{\theta_0}$ and $\norm{\vv_i} = \norm{\nabla_{\theta} f(\theta_0, \z_i)} \leq
K_{\nabla_0}$, hence the Theorem follows.

\section{Experiments}

\label{appendix:C}

We performed experiments to verify the correlation between the norm of Tangent
Sensitivity and the empirical generalization gap. We considered the task of
image classification on the well-known MNIST and CIFAR-10 datasets.
In case of the MNIST problem we trained a fully connected ReLU network,
 a Multi-layer Percpetron (MLP) of the layer structure
(768, 3000, 3000, 3000, 10). We optimized the network for squared loss with
Adam optimizer. The resulted correlation
  between the empirical generalization gap and the
average norm of the Tangent Sensitivity on the test dataset can be seen in 
Fig.~\ref{fig:fig2}. Note, as in Fig.~\ref{fig:wide} on all figures the norm 
values of tangent sensitivity were linearyly scaled for presentational purposes. 
Typically, a division by a constant with magnitude $10^5$ was enough. The class
gaps are computed by considering the $i$-th output neuron and the binary
classificaiton problem for the $i$-th class.

\begin{figure}
\centering
\includegraphics[width=0.7\textwidth]{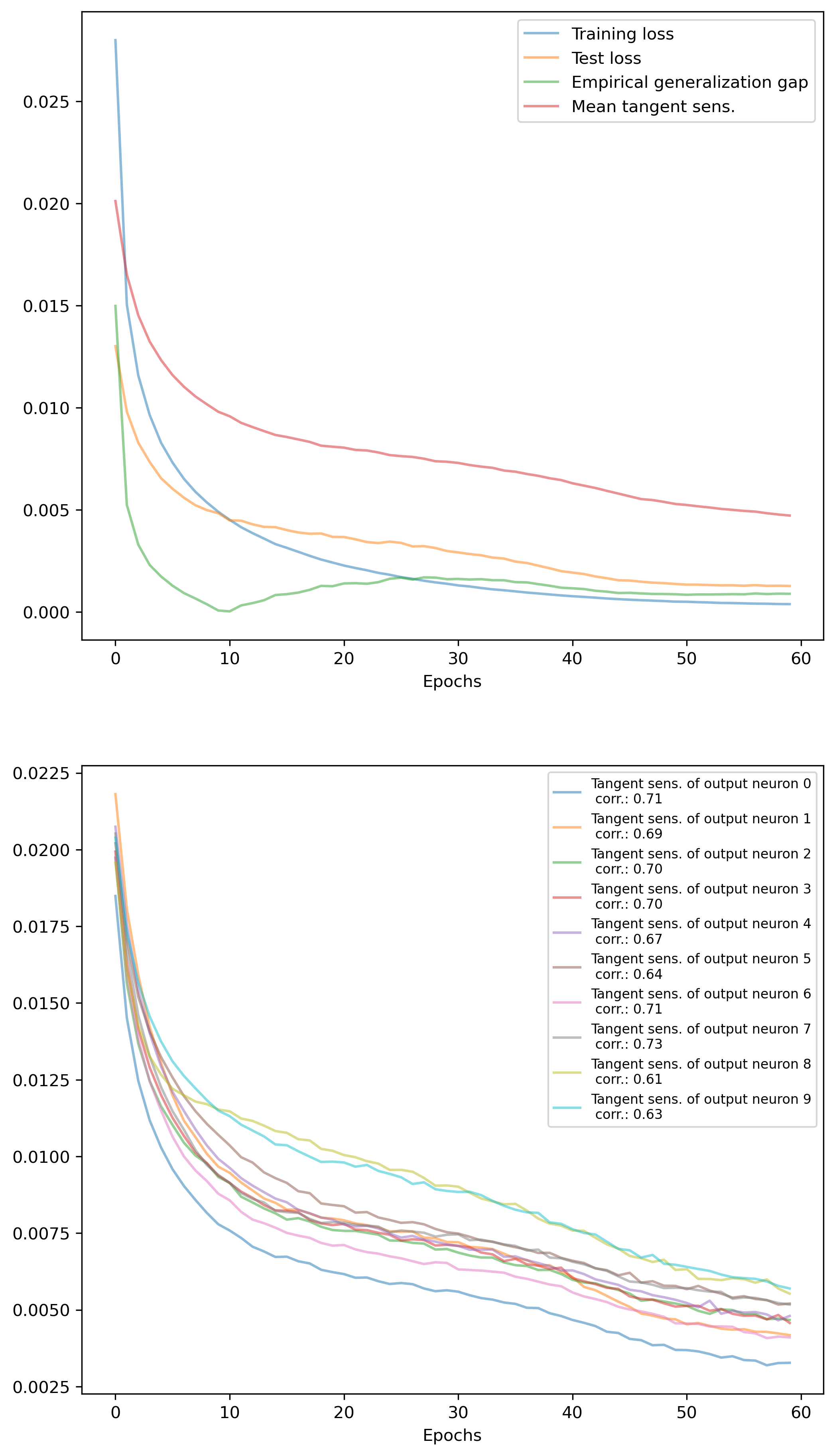}
\caption{Correlation between the empirical generalization gap and the
average norm of the Tangent Sensitivity on the test dataset
for a $3 \times 3000$-wide fully connected ReLU trained on MNIST.}
\label{fig:fig2}
\end{figure}

As for the CIFAR-10 dataset, we trained similarly structured MLPs where the three hidden layers have widths of
100, 700, 1500 and 3000. Besides the results of 3000-width network in Fig.~\ref{fig:wide}, the rest of the curves are shown in Fig.~\ref{fig:fig3}, Fig.~\ref{fig:fig4}, and Fig.~\ref{fig:fig5}. 

\begin{figure}
\centering
\includegraphics[width=0.8\textwidth]{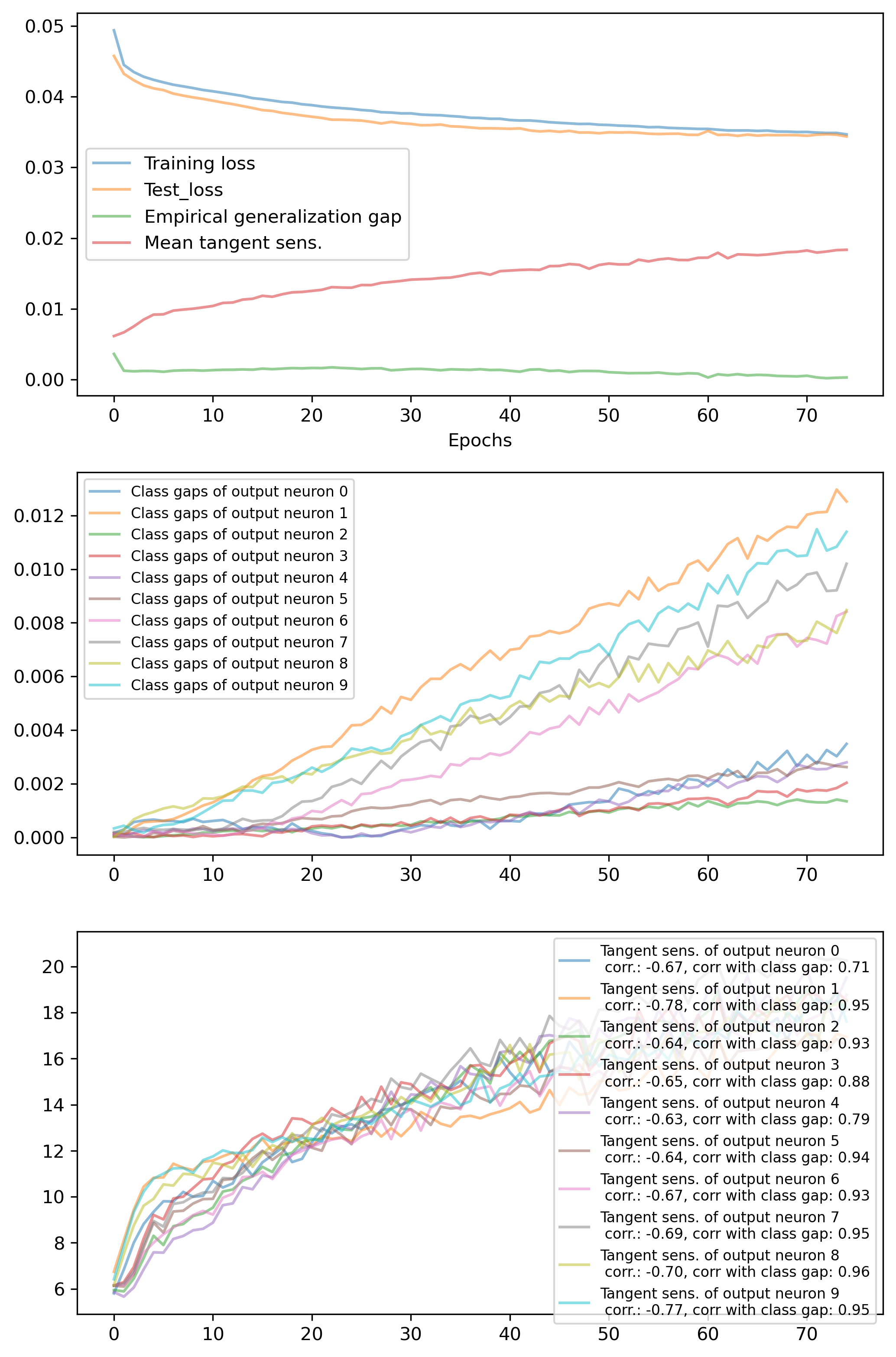}
\caption{Correlation between the empirical generalization gap and the
average norm of the Tangent Sensitivity on the test dataset
for a $3 \times 100$-wide fully connected ReLU trained on CIFAR-10.}
\label{fig:fig3}
\end{figure}

\begin{figure}
\centering
\includegraphics[width=0.8\textwidth]{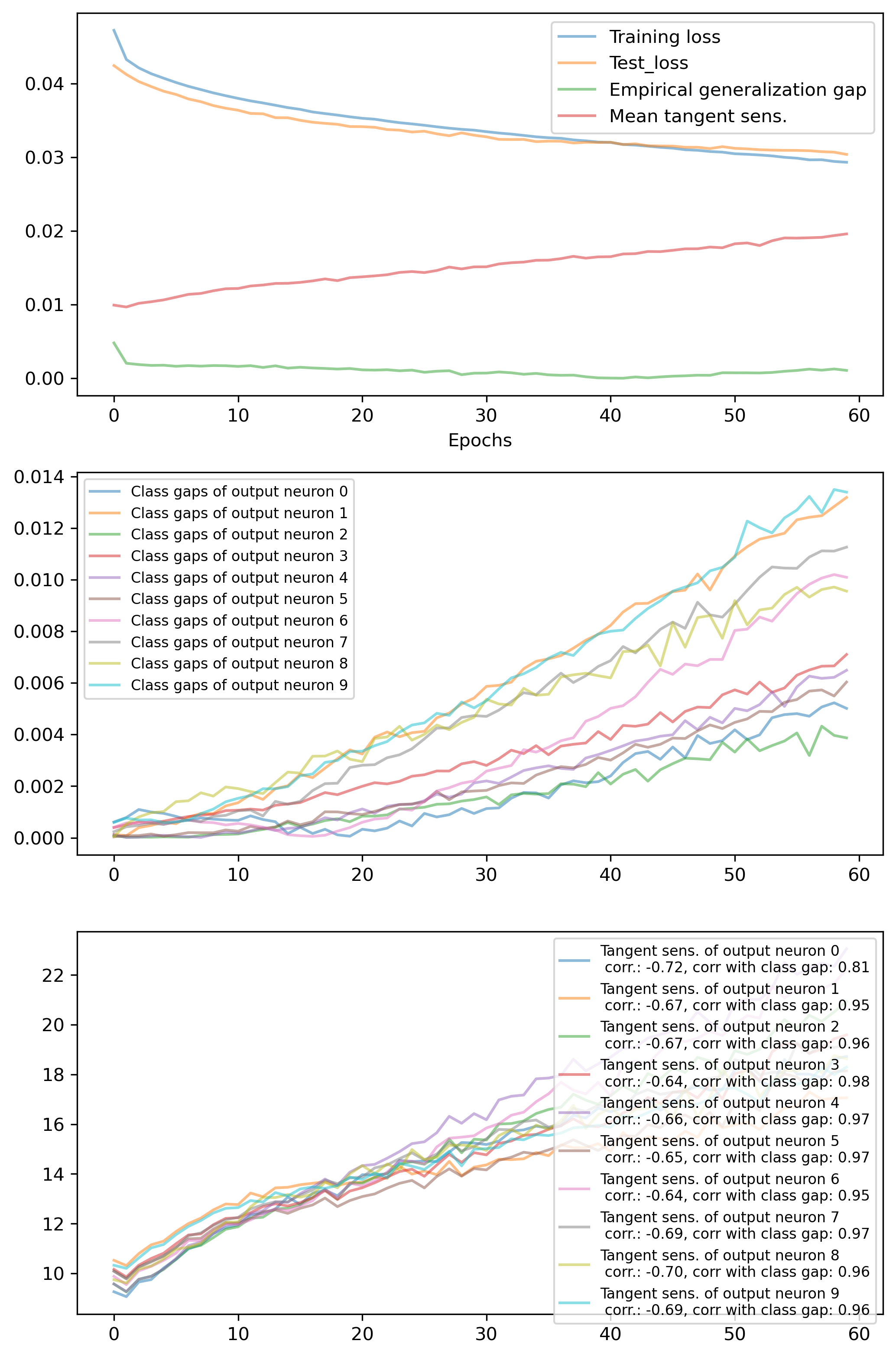}
\caption{Correlation between the empirical generalization gap and the
average norm of the Tangent Sensitivity on the test dataset
for a $3 \times 700$-wide fully connected ReLU trained on CIFAR-10.}
\label{fig:fig4}
\end{figure}

\begin{figure}
\centering
\includegraphics[width=0.8\textwidth]{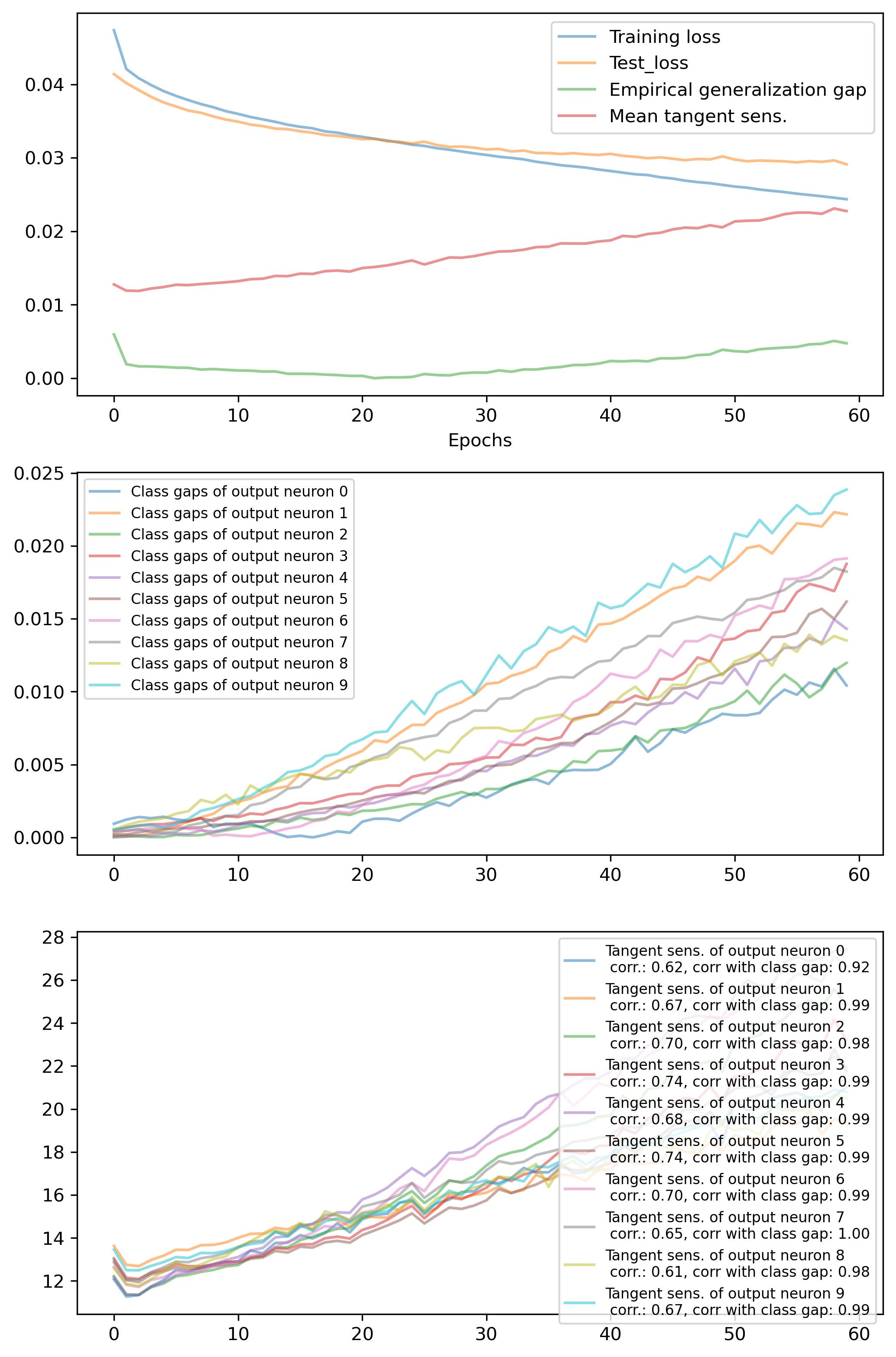}
\caption{Correlation between the empirical generalization gap and the
average norm of the Tangent Sensitivity on the test dataset
for a $3 \times 1500$-wide fully connected ReLU trained on CIFAR-10.}
\label{fig:fig5}
\end{figure}

Interestingly, in case of wide networks (3000-wide) the tangent sensitivity of different output neurons correlate with the mean generalization gap while if the network has lower width this correlation remains but only for matching per class output neuron and per class loss.  

\end{document}